\RequirePackage[OT1]{fontenc}
\documentclass[letterpaper, 10 pt, conference]{ieeeconf}
\usepackage{graphicx} 
\usepackage{orcidlink}
\usepackage{booktabs}
\usepackage{makecell}
\usepackage{xcolor}
\usepackage{amsmath} 
\usepackage{amssymb}
\usepackage{subcaption}
\usepackage{tikz}
\usepackage{twemojis}
\usepackage{dsfont}
\usepackage[capitalize]{cleveref}
\usetikzlibrary{positioning, fit}
\usetikzlibrary{3d} 
\usetikzlibrary{matrix}
\newtheorem{theorem}{Theorem}[section]

\usetikzlibrary{positioning, arrows.meta, shapes.geometric, backgrounds}

\definecolor{safetyColor}{HTML}{7ac74f}
\definecolor{biodiversitycolor}{HTML}{fb4d3d}
\definecolor{3dmodelColor}{HTML}{345995}
\definecolor{applicationColor}{HTML}{eac435}
\definecolor{backgroundColor}{HTML}{2a9d8f}
\definecolor{challengeColor}{HTML}{606c38}
\definecolor{applicationBackground}{HTML}{ffb703}

\tikzset{
 base node/.style={
 rounded corners=3pt,
 align=center,
 inner sep=3pt
 },
 category box/.style={
 rectangle,
 rounded corners,
 draw=gray!30,
 inner sep=3pt
 },
 arrow/.style={
 -latex,
 ultra thick,
 opacity=0.9
 },
 3dmodels/.style={base node, fill=white, draw=black},
 biodiversity/.style={base node, fill=biodiversitycolor!5, draw=biodiversitycolor},
 safety/.style={base node, fill=safetyColor!5, draw=safetyColor},
 applications/.style={base node, fill=white, draw=applicationColor},
 title/.style={minimum height=0cm, align=center}
}

\pgfdeclarelayer{background}
\pgfdeclarelayer{foreground}
\pgfsetlayers{background,main,foreground}

\usepackage[style=ieee,autocite=footnote,maxcitenames=1,minnames=1,maxalphanames=4,minalphanames=3,natbib=true,url=false, doi=false]{biblatex}
\DefineBibliographyStrings{english}{%
    andothers = {\normalfont et\addabbrvspace al\adddot}
}

\title{D-CAT: Decoupled Cross-Attention Transfer between Sensor Modalities for Unimodal Inference}
\author{Leen Daher, Zhaobo Wang, and Malcolm Mielle\orcidlink{0000-0002-3079-0512}
        \thanks{*This work was supported by Schindler AG}
    \thanks{Zhaobo Wang and Leen Daher are with Ecole Polytechnique Federale de Lausanne, Lausanne, Switzerland.
            {\tt\small
                @epfl.ch, @epfl.ch    }}%
    \thanks{Malcolm Mielle is with Schindler EPFL Lab, Lausanne, Switzerland
            {\tt\small malcolm.mielle@ik.me}}%
}

\begin{document}

\maketitle
\begin{abstract}

    Cross-modal transfer learning is used to improve multi-modal classification models (e.g., for human activity recognition in human-robot collaboration).
    However, existing methods require paired sensor data at both training and inference, limiting deployment in resource-constrained environments where full sensor suites are not economically and technically usable.
    To address this, we propose Decoupled Cross-Attention Transfer (D-CAT), a framework that aligns modality-specific representations without requiring joint sensor modality during inference.
    Our approach combines a self-attention module for feature extraction with a novel cross-attention alignment loss, which enforces the alignment of sensors' feature spaces without requiring the coupling of the classification pipelines of both modalities.
    We evaluate D-CAT on three multi-modal human activity datasets (IMU, video, and audio) under both in-distribution and out-of-distribution scenarios, comparing against uni-modal models.
    Results show that in in-distribution scenarios, transferring from high-performing modalities (e.g., video to IMU) yields up to $+10$\% F1-score gains over uni-modal training.
    In out-of-distribution scenarios, even weaker source modalities (e.g., IMU to video) improve target performance, as long as the target model isn't overfitted on the training data.
    By enabling single-sensor inference with cross-modal knowledge, D-CAT reduces hardware redundancy for perception systems while maintaining accuracy, which is critical for cost-sensitive or adaptive deployments (e.g., assistive robots in homes with variable sensor availability).
    Code is available at https://github.com/Schindler-EPFL-Lab/D-CAT.


\end{abstract}
\section{Introduction}

The rapid miniaturization and cost reduction of sensors have allowed significant progress in robotics, particularly in human-robot interaction, collaborative systems, and imitation learning.
By generating richer, more informative data, significant progress has been achieved in how robots can perceive, interpret, and respond to human behavior.

Yet, a critical challenge persists: while multi-sensor fusion enhances accuracy, deploying such systems at scale remains impractical.
In real-world applications---especially in industrial environments---the feasibility of equipping every operational unit or worker with extensive multi-sensor arrays (e.g., cameras, microphones, and IMUs) is limited by cost, complexity, and scalability.
Although time-limited data collection campaigns can leverage multiple sensors to curate rich training datasets, reproducing this level of sensory richness in real-world deployments is often economically or logistically infeasible.

To solve this situation, one possible solution is to develop a system that exploits multi-modal data during training but operates with only a subset of sensors at inference time.
In such a framework, each modality's classification strategy must remain independent, while knowledge is transferred between all modalities to maximize individual performance through collaborative training.
Yet, current multi-modal fusion methodologies assume the presence of all sensor modalities at both training and inference; a coupling that renders them incompatible with real-world constraints.
This disparity between what is achievable at small and large scales poses the question of how to design systems that retain the benefits of multi-modal learning, while avoiding the practical limitations of real-world robotics?

This work introduces a novel framework for decoupled knowledge transfer between arbitrary sensor modalities.
Our approach employs a modality-specific encoder together with a self-attention mechanism to capture modality-specific features.
To enable knowledge transfer between two sensor modalities, we propose a novel cross-attention loss function that optimizes the target modality's embeddings as a linear combination of the source modality's embeddings.
This ensures knowledge transfer from source to target while preserving the target modality's accuracy.

We validate our method through extensive experiments on three datasets, evaluating transfer performance across IMU, video, and audio modalities.
Our framework enables scalable sensor-fusion by leveraging multi-modal data during training while only requiring a subset of sensors at inference.

\begin{figure*}[t]{\textwidth=1cm}
  \centering
  \vspace{1mm}
  \begin{tikzpicture}[scale=1, transform shape]

    \begin{scope}[local bounding box=modality_A]

      \node[text width=1.5cm, align=center] (input_seq) at (15.35, 0) {Sensor\\signal A};

      \node [
        draw = orange,
        fill = orange!7,
        rectangle,
        rounded corners=2pt,
        minimum width=0.5cm,
        minimum height=0.5cm,
        text width=1.5cm,
        align=center,
        left = 0.3cm of input_seq
      ] (a_enc) {
        Modality A Encoder
      };

      \node [
        draw=green!30!black,
        fill = green!30!black!15!white,
        rectangle,
        rounded corners=2pt,
        minimum width=0.5cm,
        minimum height=0.5cm,
        text width=2.2cm,
        align=center,
        left = 0.5cm of a_enc
      ] (self_attn_1) {
        Self-Attention Module
      };

    \end{scope}

    \node[] at (10, 1) {\twemoji{snowflake} frozen};

    \begin{scope}[local bounding box=modality_B]

      \node[text width=1.5cm, align=center] (input_seq_2) at (-0.37, 0) {Sensor\\signal B};


      \node [
        draw= orange,
        fill = orange!7,
        rectangle,
        rounded corners=2pt,
        minimum width=0.5cm,
        minimum height=0.5cm,
        text width=1.5cm,
        align=center,
        right = 0.3cm of input_seq_2
      ] (b_enc) {
        Modality B Encoder
      };

      \node [
        draw=green!30!black,
        fill = green!30!black!15!white,
        rectangle,
        rounded corners=2pt,
        minimum width=0.5cm,
        minimum height=0.5cm,
        text width=2.2cm,
        align=center,
        right = 0.5cm of b_enc
      ] (self_attn_2) {
        Self-Attention Module
      };

    \end{scope}

    \node [
      draw,
      fill = white,
      rectangle,
      rounded corners=2pt,
      minimum width=0.5cm,
      minimum height=0.5cm,
      text width=2.2cm,
      align=center,
      left = 0.75cm of self_attn_1,
    ] (l_ca) {
      Cross-Attention Loss
    };

    \node [
      draw,
      fill = white,
      rectangle,
      rounded corners=2pt,
      minimum width=0.5cm,
      minimum height=0.5cm,
      text width=2.2cm,
      align=center,
      below = 0.1cm of l_ca,
    ] (l_ce) {
      Classification Loss
    };

    \draw[-latex] (input_seq) -- (a_enc);
    \draw[-latex]  (a_enc) -- (self_attn_1);
    \draw[-latex] (self_attn_1) -- (l_ca);
    \draw[-latex] (input_seq_2) -- (b_enc);
    \draw[-latex] (b_enc) -- (self_attn_2);
    \draw[-latex] ([xshift=0.5cm]self_attn_2.south) -- (4.775, -0.85) -- ([yshift=0.175cm]l_ce.west);
    \draw[-latex] ( self_attn_2) -- (l_ca);

    \draw[-latex, dashed] (l_ca) -- (7.5, 1) -- (4.275, 1) --(self_attn_2);
    \draw[-latex, dashed] (l_ca) -- (7.5, 1) -- (1.675, 1) --(b_enc);

    \draw[-latex, dashed] (l_ce)  -- (4.275, -1) -- (self_attn_2);
    \draw[-latex, dashed] (l_ce)  -- (1.675, -1) -- (b_enc);

    \node[] at (2.75, -1.2) {\small backpropagation};
    \node[] at (2.75, 1.2) {\small backpropagation};

    \begin{pgfonlayer}{background}
      \node[safety, fit=(modality_A)] {};
      \node[biodiversity, fit=(modality_B)] {};
    \end{pgfonlayer}

  \end{tikzpicture}
  \caption{\textbf{Cross-modal transfer architecture.} Sensor modality A (source) and B (target) classification networks each consist of an encoder and a self-attention module.
    Modality A network is pretrained and frozen during the training of modality B's classification network.
    The classification network of modality B is then trained against a standard classification loss as well as a novel cross-attention loss that aims to align the two modalities' query and value embeddings, thus enabling knowledge transfer from A to B.}
  \label{fig:cross_arch}
\end{figure*}
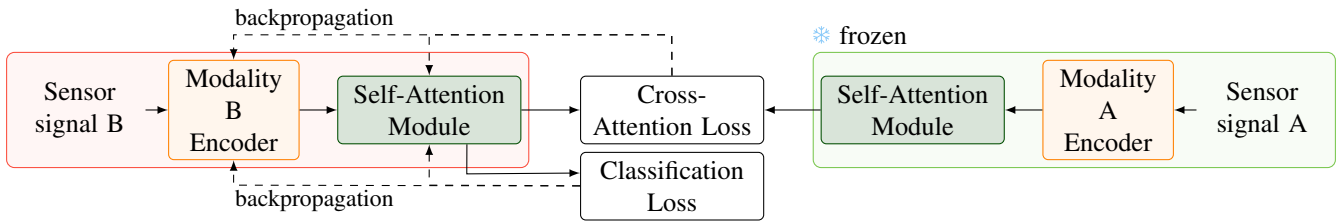

\section{Related Work}

\subsection{Human Action Recognition}
\label{sec:related:har}

Human action recognition (HAR) has been extensively studied~\cite{aggarwal2014human, wang2023comprehensive,vrigkas2015review, JAHANMAHIN2022102404} across diverse sensor modalities, driven by applications in human-robot interactions~\cite{JAHANMAHIN2022102404}, video retrieval~\cite{4130427}, or safety~\cite{haddad2025uac}.
While RGB video or image-based approaches are the most popular due to the ubiquity of cameras~\cite{wang2023comprehensive}, researchers have explored alternatives---including depth sensors~\cite{zhao2012combing}, audio sensors~\cite{Liang_2019},
or infrared thermal images~\cite{skeleton} for low-light robustness---to improve accuracy, computational cost, and privacy.

Multi-modal HAR methods frequently pairs video with depth~\cite{singh2021deeply, zhao2012combing, mukherjee2020human} (e.g., via a Kinectsensor) or audio~\cite{cristina2024audio, humactrec}, leveraging complementary strengths in both sensor modalities.
However, visual data introduces privacy risks and processing overhead, while audio---though lightweight---may lack discriminative power for fine-grained actions.
IMUs offer a low-cost, privacy-preserving alternative for HAR~\cite{zhuang2019design, hou2020study, ayman2019efficient} that is already present in many common devices (e.g., smartphones and smartwatches).
However, IMU data suffers from noise and ambiguity for similar gestures.
Furthermore, a common limitation of all the previous works using multi-modal sensor data is the reliance on coupled classification methodologies; it assumed that all sensor modalities are always available, at training and during use.


\subsection{Cross-Modal Transfer Learning}



Cross-modal transfer learning focuses on transferring knowledge between sensor modalities.
Given a source and target modality, some work~\cite{zhang2023MultiS, ma2024learningmodalityknowledgealignment,kwon2020imutubeautomaticextractionvirtual} first transform the source modality data to a feature set compatible with the target modality.
For example, \textcite{ma2024learningmodalityknowledgealignment} train a modality-specific encoder that outputs features that are directly used to pre-train a backbone model used on the target modality.
Then they refine both the encoder and the backbone model using data from the target modality, leading to higher accuracy in the predictions.
On the other hand, \textcite{kwon2020imutubeautomaticextractionvirtual} generate virtual accelerometry data from 3D human poses extracted from videos.
The virtual IMU data is used to pre-train an HAR model that is finetuned on real IMU data, yielding a 10\% improvement in accuracy.

Other methods focus on aligning the feature representations of different models~\cite{wu2024Class-, bharti2019HuMAn:, Deldari_2022, 9157122, deldari2022cocoa}.
%
\textcite{Deldari_2022} proposes a simple yet effective cross-modal contrastive loss that learns joint representations by pulling together temporally aligned samples from different sensors (positives) and pushing apart temporally distant or unrelated samples (negatives).
\textcite{9157122} propose a method that transfers knowledge between image and text models. Their method leverages both intra-and inter-modal information is implemented through self-attention and cross-attention modules, respectively.




Similarly to \cref{sec:related:har}, a common limitation of the previous works is the assumption that all sensor modalities are present at both training and inference times; it is not possible to use them for applications where only a subset of sensors would be available at inference time.
Unlike prior efforts focused on coupled modality pairs, we propose to transfer knowledge between arbitrary sensors at training time, while maintaining the ability to use a single modality at inference.

\section{Preliminaries}\label{section:prel}

In this section, we present the concepts of self-attention and cross-attention, since those are instrumental to our method.

\subsection{Self-attention}\label{section:self_att}

Self-attention~\cite{attalluneed} assigns scores that measure the importance of each element with respect to the other elements in the same sequence. These scores are calculated by mapping the embeddings of a model into queries, keys, and values, which are then used to generate new weighted embeddings.

The projection into the key, query, and value spaces is implemented using linear layers as in Eq. \ref{eq:kqv}:

\begin{equation} \label{eq:kqv}
    \begin{split}
        K & = E W_{K},\quad W_{K} \in \mathbb{R}^{d_{\text{model}}\times d_k},\quad K \in \mathbb{R}^{\text{SL}\times d_k} \\
        Q & = E W_{Q},\quad W_{Q} \in \mathbb{R}^{d_{\text{model}}\times d_k},\quad Q \in \mathbb{R}^{\text{SL}\times d_k} \\
        V & = E W_{V},\quad W_{V} \in \mathbb{R}^{d_{\text{model}}\times d_v},\quad V \in \mathbb{R}^{\text{SL}\times d_v}
    \end{split}
\end{equation}

Here, SL is the sequence's length and \(E\in\mathbb{R}^{\text{SL}\times d_{\text{model}}}\) is the embedding matrix for the input sequence of the model. The matrices $W_{K}$, $W_{Q}$, and $W_{V}$ are learnable parameters that project the embeddings into the key, query, and value spaces, respectively. $d_{model}$ is the dimension of the input embeddings and may differ across modalities, while $d_k$ and $d_v$
are the projection dimensions, referred to as $d_{out}$

The self-attention output is computed as:

\begin{equation}\label{eq:att_score}
    \text{Self-Attention}(K,Q,V) = \text{softmax}\left(\frac{Q K^\top}{\sqrt{d_{\text{out}}}}\right) V
\end{equation}

This operation produces a new sequence of the same length, where each element is a weighted sum of all elements in the original sequence.

\subsection{Cross-Attention}\label{section:cross_att}

On the other hand, cross-attention relates elements from two different sequences; the query is taken from one sequence and the keys and values from another.

Cross-attention from modality A (source) to modality B (target) is expressed as:

\begin{equation}
    \text{Cross-Attention}(K,Q,V) = \text{softmax}\left(\frac{Q_B K_A^\top}{\sqrt{d_{\text{out}}}}\right) V_A.
\end{equation}
where $Q_B$ is the target model's query matrix, and $K_A$ and $V_A$ are the source model's key and value matrices, respectively.

\section{Method}

D-CAT leverages modality-specific encoders with self-attention to extract discriminative features from individual sensor inputs.
Then, to align the target and source feature spaces, we propose a novel cross-attention loss that enforces correspondence between modalities, while preserving architectural decoupling in their downstream classification networks.
During training, the source encoder and self-attention modules remain frozen, enabling transfer to the target domain without changing the source's pretrained feature space.

While we will demonstrate our framework on three common sensor modalities in robotic tasks---inertial measurement unit (IMU) signals, video sequences, and audio signals---in \cref{sec:eval}, it should be noted that, given a modality encoder, D-CAT's design generalizes to arbitrary sensor types, offering a flexible solution for modality-agnostic knowledge transfer.
The overall framework is presented in Figure \ref{fig:cross_arch}.

\subsection{Self-Attention Module}
\label{sec:method:selfattention}

The sequences of embeddings extracted from the modality-specific encoders are used as input to a self-attention module.
The self-attention module captures dependencies within each modality as defined in \cref{section:self_att}, following \cref{eq:att_score}.
Specifically, the self-attention module provides the key, query, and value independently for each modality.
Those keys, queries, and values will be used by the proposed cross-attention loss of this paper.

\subsection{Cross-Attention Loss}



As shown in \cref{section:prel}, cross-attention relates the elements from two different sequences by using the query from one sequence and the key and value from the other sequence.
However, cross-attention creates a strong dependency between modalities, where the two sensor data cannot be decoupled at inference.
To remove this dependency, we propose a novel cross-attention loss during training to align the cross-attention embeddings ($Q K^T V$) of each modality.

Given a pretrained source classification network from which we aim to transfer knowledge to a target classification network, we first freeze the weights of the source model.
Hence, the cross-attention loss will only updates the target modality encoder and self-attention module, aligning them with the fixed source representations.

The proposed cross-attention loss aims to satisfy:
\begin{equation}\label{eq:emb_al}
    Q_BK_B^TV_B \approx Q_BK_A^TV_A
\end{equation}
with $Q_B$, $K_B$, and $V_B$ the query, key, and value projections of the target modality embeddings, and $K_A$, $V_A$ the key and value projections of the source modality embeddings.

Since the aim is to align the models' embeddings through minimizing the distance between the self-attention and cross-attention embeddings, the Frobenius norm is used to obtain the Euclidean length between each element of a matrix:
\begin{equation}\label{eq:frob}
    ||A||_F = \sqrt{\sum_{i,j}|a_{ij}|^2},
\end{equation}
where $a_{ij}$ denotes the element at the $i^{th}$ row and $j^{th}$ column.
This yields the following cross-attention loss:
\begin{equation}\label{eq:lca}
    L_{CA}=||\overline{Q_BK_B^TV_B} -\overline{Q_BK_A^TV_A}||_F
\end{equation}
where $||\overline{\cdot}||$ denotes the normalized matrices.

The query matrices can be factored out from each side, yielding a final cross-attention loss:
\begin{equation}\label{eq:lca_f}
    L_{CA}=||\overline{K_B^TV_B} -\overline{K_A^TV_A}||_F
\end{equation}

The loss aligns self-attention ($K_A^TV_A$) and cross-attention ($K_B^TV_B$) values, assuming that the softmax and scaling of \cref{eq:att_score} are dropped.
Indeed, their removal does not affect the alignment quality, since, as proven in \cref{lemma}, $K_B$ converges to a linear mapping of $K_A$, and similarly, $V_B$ converges to a linear mapping of $V_A$.
Hence, the attention terms remain comparable, ensuring that the alignment between modalities is preserved even without the softmax and scaling.

\begin{theorem}\label{lemma}
    Assume matrix M $\in \mathbb{R}^{m\times n}$ with $rank(M) = r$ where:
    \begin{equation}\label{eq:M}
        M = AB = A'B',
    \end{equation}
    where $A, A' \in \mathbb{R}^{m\times k}$ and $B, B' \in \mathbb{R}^{k\times n}$. Then there exist matrices $R \in \mathbb{R}^{k\times k'}$ and $S \in \mathbb{R}^{k'\times k}$ s.t:
    \begin{equation}
        \begin{split}
            A' & = AR \\
            B' & = SB
        \end{split}
    \end{equation}
\end{theorem}

\begin{proof}
    Since $M=AB$, every column of $M$ lies in the span of the columns of $A$. Likewise, since $M=A'B'$, every column of M lies in the span of the columns of $A'$. Hence, the column-space of M is contained in both the span of $A$ and the span of $A'$. If that is not the case, then $A'$ cannot produce the columns of $M$.



    So far, we have proved that every column of $A'$ is a linear combination of the columns of $A$. This can be reformulated as:
    \begin{equation}\label{eq:ar_reform}
        a'_j=r_{1j}a_1+r_{2j}a_2 + ...+r_{rj}a_k,
    \end{equation}
    where $a'_j$ is a column in $A'$, $(a_1, ..., a_k)$ are the columns of A, and $(r_{1j},...,r_{rj})$ are the coefficients needed to get the value of the column in $A'$. These coefficients can be collected into a matrix $R$:
    \begin{equation}\label{eq:a_prime}
        A' = AR
    \end{equation}

    Similarly, because $M =AB=A'B'$, each row of M lies in the span of the rows of B and also in the span of the rows of $B'$. Therefore, every row of $B'$ can be written as a linear combination of the rows of B. This can be written as

    \begin{equation}\label{eq:bs_reform}
        b'_i=s_{1i}b_1+s_{2i}b_2 + ...+r_{si}b_k,
    \end{equation}
    where $b'_i$ is a row in $B'$, $(b_1, ..., b_k)$ are the rows of B, and $(s_{1i},...,s_{ri})$ are the coefficients. These coefficients can be collected into a matrix $S$:
    \begin{equation}\label{eq:b_prim}
        B'=SB
    \end{equation}

    Finally, substituting \cref{eq:a_prime} and \cref{eq:b_prim} into \cref{eq:M} yields:
    \begin{equation}
        \begin{split}
            AB & = A'B'    \\
               & =(AR)(SB) \\
        \end{split}
    \end{equation}
\end{proof}


Applying \cref{lemma} to our case, where $A = K_A^T$, $A' = K_B^T$, $B = V_A$, and $B'= V_B$ yields:
\begin{equation}
    \begin{split}
         & K_B^T\to K_A^TR \\
         & V_B \to  SV_A
    \end{split}
\end{equation}

where $R$ and $S$ are the linear map matrices.

Thus, it is proven that the cross-attention loss enforces $K_B$ and $V_B$ to converge to linear mappings of $K_A$ and $V_A$, respectively, ensuring that the cross-attention term $K^T_BV_B$ remains aligned with the self-attention term $K_A^TV_A$.

\subsection{Masked Cross-Modal Alignment}\label{sec:pos_trans}
Since no classification model is a perfect oracle, all models will return erroneous classifications for some input data.
However, such incorrect classification should not be used for knowledge transfer since the source model has not been able to learn from the extracted features.

To avoid aligning embeddings of wrongly classified samples from the source modality, we use an indicator function $\mathds{1}$ with the cross-attention loss to take the cross-attention loss into account only for embeddings of correctly classified samples:
This $\mathds{1}$ is expressed as:
\begin{equation}
    \mathds{1}(x) =
    \begin{cases}
        1 & \text{if } x \text{ is a correct classification},    \\
        0 & \text{if } x \text{ is an incorrect classification}.
    \end{cases}
\end{equation}

Hence, the final cross-attention loss is expressed as:
\begin{equation}
    L_{CA}=\mathds{1}(x) \cdot ||\overline{K_B^TV_B} -\overline{K_A^TV_A}||_F,
\end{equation}


\subsection{Classification Loss}
Finally, the cross-entropy loss is adopted for classification:
\begin{equation}\label{eq:ce_loss}
    L_{CE} = - \frac{1}{N}\sum_{n=1}^Nlog(\frac{\text{exp}(x_{n,y_n})}{\sum^C_{c=1}\text{exp}(x_{n,c})})
\end{equation}
where N is the total number of samples, C is the total number of classes, $x_{n, y_n}$ is the true logit of the $n^{th}$ sample of the current batch, and $x_{n,c}$ is the model output logit for the class c.

The complete loss function is:
\begin{equation}\label{eq:full_loss}
    Loss = L_{CE} +\lambda L_{CA},
\end{equation}
where $\lambda$ is a non-negative hyper-parameter that balances the contribution of the cross-attention loss $L_{CA}$ relative to the cross-entropy loss $L_{CE}$

\section{Experiments}

In this section, we discuss the metrics used to measure the performance of knowledge transfer from source modality A to target modality B, the datasets, and the baseline model architectures.
The implementation of D-CAT and all baselines can be found online.\footnote{\url{https://github.com/Schindler-EPFL-Lab/D-CAT}}

\subsection{Datasets}

The proposed method is evaluated on three multi-modal datasets, each containing synchronized data from two sensor modalities.

1) Video-IMU: the University of Electronic Science and Technology of China Multi-Modal Egocentric Activity Dataset for Continual Learning (UESTC-MMEA-CL)~\cite{UESTC} contains synchronized Video-IMU data collected from 10 participants performing 32 everyday activities.
Each activity lasts between 8 and 35 seconds and is recorded using wearable smart glasses equipped with an RGB camera and an IMU sensor.

To reduce training complexity and hardware requirements, we use a subset of the 8 activity classes, containing 6,522 samples. These classes---drinking, reading, floor sweeping, cutting fruits, washing hands, typing on a laptop, typing on a phone, and opening/closing a door---were selected because they exhibit a high degree of variability in motion intensity, interaction with objects, and sensor signal characteristics.
For the experiment, the dataset is split into $80\%$ training, $10\%$ validation, and $10\%$ test sets.

2) Audio-IMU: the Cough Audio-IMU multimodal cough dataset using wearables~\cite{10782697} consists of synchronized IMU-Audio data recorded from 13 participants performing 8 tasks, such as coughing, sneezing, and laughing. The dataset is imbalanced, with some classes having more samples than others. To address this, samples with the shortest IMU sequences, carrying minimal motion information, were removed. After this filtering, the dataset contains 2,576 samples, split into $65\%$ training, $18\%$ validation, and $17\%$ testing.

3) Audio-Video: the Visual Geometry Group-Sound (VGGSound)~\cite{vggsound} dataset contains audio-video data sourced from YouTube covering 309 classes, each with a duration of 10 seconds. Due to the computational cost of processing the entire dataset, a random subset of 9 classes was selected for evaluation. This subset contains 2,886 clips, split into $70\%$ training, $15\%$ validation, and $15\%$ testing.

\begin{table*}[t]
    \centering
    \vspace{1.45mm}
    \caption{ Training Parameters} 
    \begin{tabular}{lccccccc}
        \toprule
        Hyper-parameter      & UESTC-IMU        & UESTC-Video       & Cough-IMU        & Cough-Audio      & VGGSound-Video   & VGGSound-Audio   \\
        \midrule
        Epochs               & 10               & 100               & 10               & 100              & 100              & 150              \\
        Learning Rate        & $5\times10^{-4}$ & $5\times10^{-4} $ & $1\times10^{-4}$ & $1\times10^{-4}$ & $1\times10^{-4}$ & $1\times10^{-4}$ \\
        Dropout              & 0.8              & 0.8               & 0.8              & 0.6              & 0.5              & 0.5              \\
        Batch Size           & 16               & 8                 & 32               & 32               & 8                & 8                \\
        Weight Decay         & 0.005            & 0.005             & 0.001            & 0.0001           & 0.005            & 0.0              \\
        Window/Stride Length & 70               & 30                & 60               & 0.5 Seconds      & 30               & 2 Seconds        \\
        Seed                 & 11               & 17                & 7                & 17               & 17               & 11               \\
        \bottomrule
    \end{tabular}
    \label{tab:hyperparams}
\end{table*}

\subsection{Data Preprocessing and Modality-Specific Encoders}

Each of the sensor pipeline starts with a modality-specific encoder to extract high-level feature embeddings specific to the sensor data.
Such embeddings will be used for the self-attention module presented in \cref{sec:method:selfattention}.
Each encoder should map the input data into sequences of embeddings of dimension ($T_m \times d_{model}$), where $d_{model}$ is model dependent.

Since this paper focuses on IMU, video, and audio data, we present a data preprocessing strategy and encoder architecture for each sensor modality.
Across all three modalities, the convolution block shown in the upper section of \cref{fig:imu_enc} serves as a fundamental building unit of the encoder backbones.


\subsubsection{IMU}

IMU signals are first min-max normalized to the range $[-1,1]$.
They are then processed by a 1D Convolutional Neural Network (CNN) backbone as proposed by \textcite{Samosa}, producing a sequence of embeddings $E^{\text{imu}} \in \mathbb{R}^{T_{\text{imu}}\times d_{\text{model}}}$.
While the overall backbone design follows the Somasa architecture, the specific parameters of the convolutional and subsequent layers (kernel sizes, strides, and number of filters) are adjusted for our datasets.
The full IMU-encoder architecture can be seen in \cref{fig:imu_enc}.

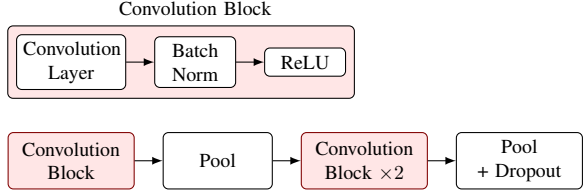
\begin{figure}[t]
  \centering
  \begin{tikzpicture}[scale=0.75, transform shape]
    \begin{scope}[local bounding box=conv_block]
      \node[
        draw,
        fill = white,
        rectangle,
        rounded corners=2pt,
        minimum width=0.5cm,
        minimum height=0.5cm,
        text width=1.7cm,
        align=center,
      ] (conv1) at (0, 1.75) {
        Convolution Layer
      };
      \node [
        draw,
        fill = white,
        rectangle,
        rounded corners=2pt,
        minimum width=0.5cm,
        minimum height=0.5cm,
        text width=1.2cm,
        align=center,
        right = 0.5cm of conv1,
      ] (bn) {
        Batch Norm
      };
      \node [
        draw,
        fill = white,
        rectangle,
        rounded corners=2pt,
        minimum width=0.5cm,
        minimum height=0.5cm,
        text width=1.2cm,
        align=center,
        right = 0.5cm of bn,
      ] (relu) {
        ReLU
      };

      \draw[-latex] (conv1) -- (bn);
      \draw[-latex] (bn) -- (relu);
    \end{scope}

    \node[above=0.25cm of bn] {Convolution Block};
    \node [
      draw=  red!50!black,
      fill = red!10,
      rectangle,
      rounded corners=2pt,
      minimum width=0.5cm,
      minimum height=1cm,
      text width=2cm,
      align=center,
    ] (convblock) {
      Convolution Block
    };
    \node [
      draw,
      fill = white,
      rectangle,
      rounded corners=2pt,
      minimum width=0.5cm,
      minimum height=1cm,
      text width=1.7cm,
      align=center,
      right = 0.5cm of convblock,
    ] (pool1) {Pool};

    \node [
      draw= red!50!black,
      fill = red!10,
      rectangle,
      rounded corners=2pt,
      minimum width=0.5cm,
      minimum height=1cm,
      text width=2cm,
      align=center,
      right = 0.5cm of pool1 ,
    ] (block3) {
      Convolution Block $\times 2$
    };
    \node [
      draw,
      fill = white,
      rectangle,
      rounded corners=2pt,
      minimum width=0.5cm,
      minimum height=1cm,
      text width=2cm,
      align=center,
      right = 0.5cm of block3 ,
    ] (pool2) {
      Pool \\+ Dropout
    };

    \draw[-latex] (convblock) -- (pool1);
    \draw[-latex] (pool1) -- (block3);
    \draw[-latex] (block3) -- (pool2);

    \begin{pgfonlayer}{background}
      \node[draw, rounded corners=2pt, fill=red!10, fit=(conv_block)] {};
    \end{pgfonlayer}

  \end{tikzpicture}
  \caption{\textbf{IMU encoder architecture}. The model begins with a convolution layer, batch normalization, and ReLU activation, followed by a pooling layer. Two additional convolutional blocks are applied, and the encoder concludes with a pooling and dropout stage.}
  \label{fig:imu_enc}
\end{figure}

\subsubsection{Video}
Each frame of a video sequence is first resized to a fixed resolution $(224, 224)$.
The sequence is then passed through a 2D-CNN based backbone Image ResNet-101 that extracts features from each frame.
The full architecture is shown in \cref{fig:img_enc}.


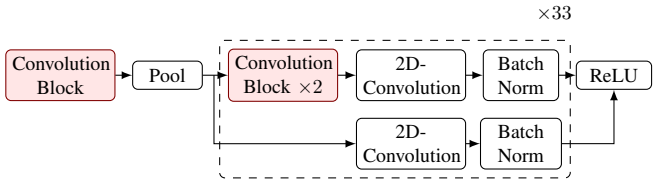
\begin{figure}[t]

  \begin{tikzpicture}[scale=0.75, transform shape]


    \node [
      draw=  red!50!black,
      fill = red!10,
      rectangle,
      rounded corners=2pt,
      minimum width=0.5cm,
      minimum height=0.5cm,
      text width=1.7cm,
      align=center,
    ] (block1) {
      Convolution Block
    };
    \node[
      draw,
      fill = white,
      rectangle,
      rounded corners=2pt,
      minimum width=0.5cm,
      minimum height=0.5cm,
      text width=1.0cm,
      align=center,
      right = 0.3cm of block1
    ] (pool) {
      Pool
    };

    \node at (8.75, 1.1) {\bfseries $\times 33$};
    \begin{scope}[local bounding box=times]
      \node [
        draw= red!50!black,
        fill = red!10,
        rectangle,
        rounded corners=2pt,
        minimum width=0.5cm,
        minimum height=0.5cm,
        text width=1.7cm,
        align=center,
        right = 0.45cm of pool
      ] (block2) {
        Convolution Block $\times 2$
      };
      \node[
        draw,
        fill = white,
        rectangle,
        rounded corners=2pt,
        minimum width=0.5cm,
        minimum height=0.5cm,
        text width=1.7cm,
        align=center,
        right = 0.32cm of block2,
      ] (conv1) {
        2D-Convolution
      };
      \node [
        draw,
        fill = white,
        rectangle,
        rounded corners=2pt,
        minimum width=0.5cm,
        minimum height=0.5cm,
        text width=1.1cm,
        align=center,
        right = 0.3cm of conv1,
      ] (bn1) {
        Batch Norm
      };

      \node[
        draw,
        fill = white,
        rectangle,
        rounded corners=2pt,
        minimum width=0.5cm,
        minimum height=0.5cm,
        text width=1.7cm,
        align=center,
        below = 0.3cm of conv1,
      ] (conv2) {
        2D-Convolution
      };
      \node [
        draw,
        fill = white,
        rectangle,
        rounded corners=2pt,
        minimum width=0.5cm,
        minimum height=0.5cm,
        text width=1.2cm,
        align=center,
        below = 0.3cm of bn1,
      ] (bn2) {
        Batch Norm
      };
    \end{scope}

    \node [
      draw,
      fill = white,
      rectangle,
      rounded corners=2pt,
      minimum width=0.5cm,
      minimum height=0.5cm,
      text width=1.1cm,
      align=center,
      right = 0.3cm of bn1,
    ] (relu) {
      ReLU
    };

    \draw[-latex] (block1) -- (pool);
    \draw[-latex] (pool) -- (block2);
    \draw[-latex] (block2) -- (conv1);
    \draw[-latex] (conv1) -- (bn1);
    \draw[-latex] (bn1) -- (relu);
    \draw[-latex] ([xshift = 2mm] pool.east) |- (conv2);
    \draw[-latex] (conv2) -- (bn2);
    \draw[-latex] (bn2.east) -| (relu.south);


    \begin{pgfonlayer}{background}
      \node[draw, rounded corners=2pt, dashed, fit=(times)] {};
    \end{pgfonlayer}

  \end{tikzpicture}
  \caption{\textbf{Image encoder architecture}. The model begins with an initial convolution block and pooling layer, followed by two stacked convolution blocks. A residual block, repeated 33 times, combines pairs of 2D convolutions with batch normalization and a ReLU activation.}
  \label{fig:img_enc}
\end{figure}

\subsubsection{Audio}

Audio signals are first converted to time-frequency representations using Short-Time Fourier (STFT).
The resulting spectrograms are passed through a Mel filter bank to get the Mel spectrogram that represents the energy levels for each of the frequency bands.
Finally, to obtain a more stable signal, the logarithm of the Mel spectrogram is calculated and flattened to be used as input data.

The processed input sequences are fed to a Pretrained audio neural network (Panns) backbone architecture from \textcite{kong2020pannslargescalepretrainedaudio}---see \cref{fig:audio_enc}.

\begin{figure}[t]
  \begin{tikzpicture}[scale=0.775, transform shape]

    \node [
      draw,
      fill = white,
      rectangle,
      rounded corners=2pt,
      minimum width=0.5cm,
      minimum height=0.5cm,
      text width=2.0cm,
      align=center
    ] (aug) {
      Augmented Spectrogram
    };

    \node at (5, 1.1) {\bfseries $\times 6$};
    \begin{scope}[local bounding box=times]
      \node [
        draw= red!50!black,
        fill = red!10,
        rectangle,
        rounded corners=2pt,
        minimum width=0.5cm,
        minimum height=0.5cm,
        text width=1.7cm,
        align=center,
        right = 0.4cm of aug
      ] (block1) {
        Convolution Block $\times 2$
      };
      \node [
        draw,
        fill = white,
        rectangle,
        rounded corners=2pt,
        minimum width=0.5cm,
        minimum height=0.5cm,
        text width=1.15cm,
        align=center,
        right = 0.3cm of block1
      ] (pool) {
        Pool \\ + \\ Dropout
      };
    \end{scope}

    \node [
      draw,
      fill = white,
      rectangle,
      rounded corners=2pt,
      minimum width=0.5cm,
      minimum height=0.5cm,
      text width=2.4cm,
      align=center,
      right = 0.4cm of pool
    ] (fc) {
      Fully-Connected Layer
    };
    \node [
      draw,
      fill = white,
      rectangle,
      rounded corners=2pt,
      minimum width=0.5cm,
      minimum height=0.5cm,
      text width=1.15cm,
      align=center,
      right = 0.3cm of fc
    ] (relu) {
      ReLU \\ + \\ Dropout
    };

    \draw[-latex] (aug) -- (block1);
    \draw[-latex] (block1) -- (pool);
    \draw[-latex] (pool) -- (fc);
    \draw[-latex] (fc) -- (relu);


    \begin{pgfonlayer}{background}
      \node[draw, rounded corners=2pt, dashed, fit=(times)] {};
    \end{pgfonlayer}

  \end{tikzpicture}
  \caption{\textbf{Audio encoder architecture}. The model processes augmented spectrograms through two consecutive convolutional blocks, followed by a sequence of pooling and dropout layers repeated six times, and finally a fully connected layer with ReLU and dropout.}
  \label{fig:audio_enc}
\end{figure}
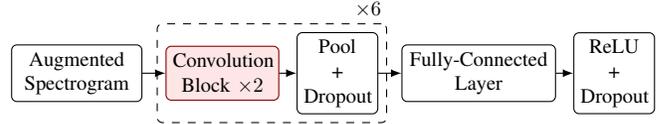

\subsection{Implementation Details}
Unless otherwise stated, models are trained using the Adam optimizer with an attention size of 512 and ReLU activation functions.
Dropout and weight decay were used for regularization and
optimized using a grid search; \cref{tab:hyperparams} summarizes the exact value of each hyper-parameters for all experiments on a given dataset.

\subsection{Evaluation Metrics}

To comprehensively assess the model performance, we measure 1) the \textit{accuracy} to measure the proportion of correctly predicted samples over the total number of samples, 2) the \textit{precision} to measure the proportion of predicted positive samples that are actually correct, reflecting the model's ability to avoid false positives, 3) the \textit{recall} to measure the proportion of actual positive samples that are correctly identified, indicating the model's ability to avoid false negatives and detect only relevant samples, and 4) the \textit{F1-score}, which is the harmonic mean of precision and recall, providing a balanced measure that accounts for both false positives and false negatives.

\section{Results}
\label{sec:eval}

We first present the overall performance across all datasets and modalities for in-distribution (ID) and out-of-distribution (OOD) scenarios.
In the ID scenario, data subjects can be present in the training, validation, and test sets.
On the other hand, in OOD scenarios, the training, validation, and test sets have no overlapping subjects,
making this split more challenging than ID, but also closer to usual real-world scenarios.
For more details on the ID and OOD scenarios, see \textcite{haddad2025uac}.

To evaluate knowledge transfer, our method is compared against uni-modal baselines.
We also provide ablation studies to analyze the impact of the masked cross-modal alignment of \cref{sec:pos_trans} (denoted as MCMA in the experiments), and of the transfer knowledge weight factor ($\lambda$).

\subsection{Comparison to Uni-modal Baseline Models}


\subsubsection{In-distribution Scenario}
\label{sec:eval:id}

The results for the ID scenario are presented in \cref{tab:indist}.
As hypothesized, knowledge transfer from more informative sensor modalities (video and audio) improved the performance of the less expressive sensor modality (IMU) on all metrics---our proposed method outperforms the corresponding uni-modal IMU models in accuracy, recall, precision, and F1-score.
For example, UESTC-IMU accuracy increased by $9\%$, its recall by $8\%$, its precision by $6\%$, and its F1-score by $8\%$.
Similarly, Cough-IMU improves its accuracy by $4\%$, its recall by $3\%$, its precision by $6\%$, and its F1-score by $6\%$.

On the other hand, one can see that transferring knowledge from the less expressive sensor modality to the more expressive ones does not lead to improvements and can even create a slight decrease in performance.
This is observed for both the image and audio models in \cref{tab:indist}.

In conclusion, in the ID scenario, our method helps improve the general performance of IMU-based model but does not generate a positive outcome when trying to transfer knowledge from IMU to other, more expressive, sensor modalities.
Hence, given two sensor modalities, our method can be used to improve the performance of lower-quality model using a better-performing model, without coupling both sensor modality for inference.

\begin{table}[t]
    \centering
    \vspace{1.45mm}
    \caption{\textbf{In-distribution results.} Performance comparison between baseline models and the proposed transfer learning framework with and without masked alignment (positive transfer) across all datasets and modalities.
        Since VGGSound does not have annotations for users, we could only run the experiment on this dataset for out-of-distribution (see \cref{tab:ood}).
        The baseline for IMU data correspond to the Samosa architecture~\cite{Samosa}, Resnet-101 for images, and \textcite{kong2020pannslargescalepretrainedaudio} for audio data.
    }
    \label{tab:indist}
    \resizebox{\columnwidth}{!}{%
        \begin{tabular}{lllcccc}
            \toprule
            Dataset & Modality & Metric               & Baseline       & \makecell{Ours} & \makecell{No      \\MCMA}\\
            \midrule
            UESTC   & IMU      & Accuracy $\uparrow$  & 0.877          & \textbf{0.969}  & \underline{0.914} \\
                    &          & Recall $\uparrow$    & 0.895          & \textbf{0.967}  & \underline{0.918} \\
                    &          & Precision $\uparrow$ & 0.900          & \textbf{0.974}  & \underline{0.913} \\
                    &          & F1-Score$\uparrow$   & 0.893          & \textbf{0.967}  & \underline{0.914} \\
            \cmidrule(lr){2-6}
                    & Image    & Accuracy $\uparrow$  & \textbf{0.957} & 0.939           & \underline{0.945} \\
                    &          & Recall $\uparrow$    & \textbf{0.953} & 0.922           & \underline{0.934} \\
                    &          & Precision $\uparrow$ & \textbf{0.957} & 0.942           & \underline{0.955} \\
                    &          & F1-Score$\uparrow$   & \textbf{0.954} & 0.925           & \underline{0.943} \\
            \midrule
            Cough   & IMU      & Accuracy $\uparrow$  & 0.440          & \textbf{0.474}  & \underline{0.463} \\
                    &          & Recall $\uparrow$    & 0.268          & \textbf{0.304}  & \underline{0.297} \\
                    &          & Precision $\uparrow$ & 0.185          & \textbf{0.243}  & \underline{0.238} \\
                    &          & F1-Score$\uparrow$   & 0.205          & \textbf{0.256}  & \underline{0.251} \\
            \cmidrule(lr){2-6}
                    & Audio    & Accuracy $\uparrow$  & \textbf{0.839} & 0.816           & \underline{0.825} \\
                    &          & Recall $\uparrow$    & \textbf{0.757} & 0.664           & \underline{0.683} \\
                    &          & Precision $\uparrow$ & \textbf{0.793} & 0.685           & \underline{0.694} \\
                    &          & F1-Score$\uparrow$   & \textbf{0.745} & 0.670           & \underline{0.685} \\

            \bottomrule
        \end{tabular}
    }
\end{table}

\subsubsection{Out-of-distribution Scenario}

the OOD scenario introduces a greater generalization challenge than the ID scenario since each model is tested on subjects unseen during training.
Looking at the results in \cref{tab:ood}, unlike the ID scenario, knowledge transfer in the OOD case does not follow the pattern of the stronger modality consistently improving the weaker one.

For the UESTC dataset, no gain is observed for the IMU modality, while the image modality shows a clear improvement of $8\%$ in accuracy, $8\%$ in recall, $9\%$ in precision, and $9\%$ for the F1-score.
A similar result is observed for the Cough and VGGSound dataset.
For the Cough and VGG Sound dataset, the IMU and image classification models only see marginal improvements from knowledge transfer from the IMU modality.
On the other hand, the audio classification model performance improves thanks to knowledge transfer from the IMU modality---$7\%$ in accuracy, $1\%$ in recall, $9\%$ in precision, and $10\%$ for the F1-score, on the Cough dataset, and $3\%$ in accuracy, $5\%$ in recall, $2\%$ in precision, and $5\%$ for the F1-score, on VGGSound.

Through our experiments, we observed that the IMU and Image models overfitted on the training dataset.
We performed a grid search of hyperparameters, including varying learning rate between \{0.0005, 0.0001, 0.001, 0.005\}, window size for IMU models between \{20, 50, 70\}, dropout rate between \{0.5, 0.6, 0,7, 0.8\}, and various train, test and validation splits, and were unable to not overfit the models and reach higher and similar metrics to the ones presented in \cref{tab:ood}.
While overfitting didn't impact the expected knowledge transfer in \cref{sec:eval:id}, we hypothesize that this overfitting on the training dataset led to the current results in OOD scenarios.
Indeed, our hypothesis is that, while an overfitted model is still able to provide meaningful knowledge to a target model (as in the case of IMU to Image in UESTC, IMU to Audio in Cough, and Image to Audio in VGGSound) in ID scenarios where user are both in the train and test sets, the same overfitted model cannot help improving the generalization in OOD scenario since it is overfit to the specific users in the training set.

Overall, these results suggest that, in OD scenarios, even higher performing classification models benefit from knowledge transfer from lower performing models, demonstrating that cross-modal transfer can improve performances under distribution shift.
While less accurate models are not able to transfer knowledge in ID scenario---since most of the information is contained in the training set---in OOD scenarios, the added knowledge provided by the weaker modality help generalize the capability of stronger classification models.



\begin{table}[t]
    \centering
    \vspace{1.45mm}
    \caption{\textbf{Out-of-distribution results.} Performance comparison between baseline models and the proposed transfer learning framework with and without masked alignment (positive transfer) across all datasets and modalities.
        The baseline for IMU data correspond to the Samosa architecture~\cite{Samosa}, Resnet-101 for images, and \textcite{kong2020pannslargescalepretrainedaudio} for audio data.
    }
    \resizebox{\columnwidth}{!}{%
        \begin{tabular}{lllccc}
            \toprule
            Dataset & Modality & Metric               & Baseline          & \makecell{Ours}   & \makecell{No      \\MCMA}\\
            \midrule
            UESTC   & IMU      & Accuracy $\uparrow$  & \textbf{0.554}    & \underline{0.470} & 0.411             \\
                    &          & Recall $\uparrow$    & \textbf{0.556}    & \underline{0.467} & 0.414             \\
                    &          & Precision $\uparrow$ & \textbf{0.554}    & \underline{0.475} & 0.449             \\
                    &          & F1-Score$\uparrow$   & \textbf{0.522}    & \underline{0.450} & 0.385             \\
            \cmidrule(lr){2-6}
                    & Image    & Accuracy $\uparrow$  & 0.518             & \textbf{0.601}    & \underline{0.548} \\
                    &          & Recall $\uparrow$    & 0.518             & \textbf{0.595}    & \underline{0.550} \\
                    &          & Precision $\uparrow$ & 0.540             & \textbf{0.627}    & \underline{0.554} \\
                    &          & F1-Score$\uparrow$   & 0.508             & \textbf{0.597}    & \underline{0.543} \\

            \midrule
            Cough   & IMU      & Accuracy $\uparrow$  & 0.416             & \textbf{0.423}    & \textbf{0.423}    \\
                    &          & Recall $\uparrow$    & \textbf{0.309}    & \textbf{0.309}    & \textbf{0.309}    \\
                    &          & Precision $\uparrow$ & 0.250             & \underline{0.259} & \textbf{0.267}    \\
                    &          & F1-Score$\uparrow$   & \textbf{0.250}    & \underline{0.249} & 0.248             \\
            \cmidrule(lr){2-6}
                    & Audio    & Accuracy $\uparrow$  & 0.607             & \underline{0.680} & \textbf{0.709}    \\
                    &          & Recall $\uparrow$    & 0.533             & \underline{0.618} & \textbf{0.653}    \\
                    &          & Precision $\uparrow$ & 0.477             & \underline{0.572} & \textbf{0.686}    \\
                    &          & F1-Score$\uparrow$   & 0.482             & \underline{0.578} & \textbf{0.636}    \\

            \midrule
            VGG-    & Image    & Accuracy $\uparrow$  & \underline{0.459} & 0.456             & \textbf{0.466}    \\
            Sound   &          & Recall $\uparrow$    & \textbf{0.330}    & 0.324             & \underline{0.325} \\
                    &          & Precision $\uparrow$ & \textbf{0.309}    & \underline{0.285} & 0.267             \\
                    &          & F1-Score$\uparrow$   & \underline{0.294} & \textbf{0.297}    & 0.285             \\
            \cmidrule(lr){2-6}
                    & Audio    & Accuracy $\uparrow$  & 0.897             & \underline{0.926} & \textbf{0.928}    \\
                    &          & Recall $\uparrow$    & 0.828             & \underline{0.883} & \textbf{0.891}    \\
                    &          & Precision $\uparrow$ & 0.889             & \underline{0.909} & \textbf{0.910}    \\
                    &          & F1-Score$\uparrow$   & 0.839             & \underline{0.892} & \textbf{0.898}    \\

            \bottomrule
        \end{tabular}
    }
    \label{tab:ood}
\end{table}

\subsection{Masked Cross-Modal Alignment}

Tables~\ref{tab:indist} and~\ref{tab:ood} report the results with and without masked alignment. As described in Section~\ref{sec:pos_trans}, our method refers to the masked alignment strategy, where knowledge transfer is restricted to samples that are correctly classified by the source modality. In contrast, no MMCA denotes the case where knowledge is transferred across all samples, regardless of correctness

In the in-distribution split, whenever knowledge transfer is effective, our method consistently outperforms no MMCA, as expected. However, in the OOD scenario, our method does not always yield better performance. For example, in the Cough-Audio dataset and in both VGGSound modalities, no MMCA slightly outperforms our method. Nevertheless, these differences are marginal; looking at the models that are improved by cross-modal alignment, it is $5\%$ for Cough-Audio metrics, less than $1\%$ for all metrics of VGGSound-Audio, and $5\%$ on the metrics of UESTC-Image.
Hence, in OOD scenarios, masked cross-modal alignment should be used on a per-dataset basis.

\subsection{Transfer Knowledge Weight Factor}

The transfer knowledge weight factor $\lambda$, is an essential hyper-parameter of the objective function \cref{eq:full_loss}.
To evaluate the impact of this parameter, we conduct an ablation study where $\lambda$ is varied over the set $\{0.1, 1.0, 10\}$.
Results for all datasets, on the more challenging OOD scenario, are shown in \cref{tab:lambda}.

Results demonstrate that while $\lambda$ has some impact on the results, especially in higher values (e.g., $\lambda=10$), the results are best at around $\lambda=1$---with little variation between $0.01$ and $1$---apart from UESTC-IMU.
Hence, in previous experiments (\cref{tab:indist} and~\ref{tab:ood}) $\lambda=1$ was used.

\begin{table}[t]
    \centering
    \vspace{1.45mm}
    \caption{\textbf{Effect of the cross-attention weight factor $\lambda$.} Results on the OOD split for all datasets and modalities when varying the cross-attention weight $\lambda \in \{0.01, 0.1, 1.0, 10\}$}
    \resizebox{\columnwidth}{!}{%
        \begin{tabular}{lllccccc}
            \toprule
            Dataset & Modality & $\lambda$         & Accuracy          & Recall            & Precision         & \makecell{F1-     \\Score}      \\
            \midrule
            UESTC   & IMU      & $  0.01$          & \underline{0.524} & \underline{0.519} & \underline{0.519} & \underline{0.506} \\
                    &          & \textbf{$  0.1$}  & \textbf{0.566}    & \textbf{0.557}    & \textbf{0.537}    & \textbf{0.534}    \\
                    &          & $  1.0$           & 0.470             & 0.467             & 0.475             & 0.450             \\
                    &          & $  10$            & 0.321             & 0.304             & 0.176             & 0.197             \\
            \cmidrule(lr){2-7}
                    & Image    & $  0.01$          & 0.470             & 0.473             & \underline{0.536} & 0.480             \\
                    &          & $  0.1$           & 0.427             & 0.427             & 0.472             & 0.416             \\
                    &          & \textbf{$  1.0$}  & \textbf{0.601}    & \textbf{0.600}    & \textbf{0.627}    & \textbf{0.597}    \\
                    &          & $  10$            & \underline{0.518} & \underline{0.513} & 0.531             & \underline{0.506} \\
            \midrule
            Cough   & IMU      & $  0.01$          & \underline{0.421} & 0.303             & \underline{0.251} & \underline{0.239} \\
                    &          & $  0.1$           & \underline{0.421} & \underline{0.304} & 0.242             & \underline{0.239} \\
                    &          & \textbf{$  1.0$}  & \textbf{0.423}    & \textbf{0.309}    & \textbf{0.259}    & \textbf{0.249}    \\
                    &          & $  10$            & 0.416             & 0.297             & 0.226             & 0.227             \\
            \cmidrule(lr){2-7}
                    & Audio    & \textbf{$  0.01$} & \textbf{0.686}    & \textbf{0.653}    & \textbf{0.728}    & \textbf{0.608}    \\
                    &          & $  0.1$           & 0.652             & 0.597             & 0.578             & 0.563             \\
                    &          & $  1.0$           & \underline{0.680} & \underline{0.618} & 0.572             & \underline{0.578} \\
                    &          & $  10$            & 0.611             & 0.539             & \underline{0.598} & 0.495             \\
            \midrule

            VGG-    & Image    & $  0.01$          & 0.452             & \textbf{0.345}    & \underline{0.295} & \textbf{0.315}    \\
            Sound   &          & $  0.1$           & \textbf{0.473}    & \underline{0.337} & 0.262             & 0.293             \\
                    &          & $  1.0$           & \underline{0.456} & 0.324             & 0.285             & \underline{0.297} \\
                    &          & $  10$            & 0.452             & 0.326             & \textbf{0.364}    & 0.291             \\
            \cmidrule(lr){2-7}
                    & Audio    & \textbf{$  0.01$} & \textbf{0.937}    & \textbf{0.910}    & \textbf{0.938}    & \textbf{0.916}    \\
                    &          & $  0.1$           & 0.911             & 0.859             & 0.892             & 0.870             \\
                    &          & $  1.0$           & \underline{0.926} & \underline{0.883} & \underline{0.909} & \underline{0.892} \\
                    &          & $  10$            & 0.896             & 0.811             & 0.850             & 0.812             \\
            \bottomrule
        \end{tabular}}
    \label{tab:lambda}
\end{table}

\section{Conclusion and future work}

While HAR often relies on multiple sensors, deploying such multi-modal sensor system in real-world settings and at scale is often infeasible due to cost or technical constraints.
To address this, we present D-CAT, a decoupled cross-attention transfer learning framework that aligns modality-specific feature spaces during training while enabling single-modality inference.
Unlike prior work that fuses modalities via cross-attention layers, D-CAT enforces alignment through a cross-attention loss, ensuring that each modality's classifier remains independent and eliminates the need for paired inputs.

We evaluate D-CAT on three HAR benchmarks using IMU, video, and audio.
When transferring knowledge from a high-accuracy source modality to a weaker target classifier, D-CAT improves F1 scores by up to $7$\% in in-distribution settings (train-test subject overlap).

Furthermore, in out-of-distribution evaluation (unseen users in test set), D-CAT even enables knowledge transfer from low to high accuracy modalities, suggesting that its decoupled alignment may improve robustness to domain shift.
These results demonstrate that D-CAT offers a scalable alternative to joint training, particularly for resource-constrained applications where sensor availability is limited.

However, D-CAT's effectiveness decreases when the target model is already overfitted, limiting its utility in some out-of-distribution scenarios.
Future work will explore strategies to mitigate this limitation, as well as explore the use of more than one source modality for transfer.
Overall, our findings suggest that decoupled cross-attention could unlock modular and efficient knowledge transfer for more-accurate HAR with limited sensors.

\printbibliography
\end{document}